\documentclass[letterpaper, 10 pt, conference]{ieeeconf} 
\IEEEoverridecommandlockouts                              

\usepackage{graphics} 
\usepackage{epsfig} 
\usepackage{times} 
\usepackage{amsmath} 
\usepackage{amssymb}  
\usepackage{color}
\usepackage{comment}
\usepackage{algorithmic, algorithm} 
\usepackage{cite} 

\def\aspeed{\bar{v}_A}

\newtheorem{theorem}{Theorem}

\newtheorem{remark}{Remark}
\newtheorem{lemma}{Lemma}
\newtheorem{definition}{Definition}
\def\bs{\boldsymbol}
\def\mf{\mathbf}

\def\mc{\mathcal}
\def\beq{\begin{equation*}}
\def\eeq{\end{equation*}}
\def\bql{\begin{equation}}
\def\eql{\end{equation}}
\def\bqn{\begin{eqnarray*}}
\def\eqn{\end{eqnarray*}}
\def\bnl{\begin{eqnarray}}
\def\enl{\end{eqnarray}}
\def\bna{\bql\begin{array}{rcl}}
\def\ena{\end{array}\eql}
\def\bnn{\beq\begin{array}{rcl}}
\def\enn{\end{array}\eeq}
\def\bma{\begin{bmatrix}}
\def\ema{\end{bmatrix}}
\def\bmx{\begin{matrix}}
\def\emx{\end{matrix}}
\def\ben{\begin{enumerate}}
\def\een{\end{enumerate}}
\def\bit{\begin{itemize}}
\def\eit{\end{itemize}}
\def\bei{\begin{itemize}}
\def\eei{\end{itemize}}
\def\bet{\begin{tabular}}
\def\eet{\end{tabular}}

\newcommand{\allcaps}[1]{\uppercase\expandafter{#1}}

\def\bfs{\begin{footnotesize}}
\def\efs{\end{footnotesize}}
\def\bss{\begin{small}}
\def\ess{\end{small}}

\def\ds{\displaystyle}

\def\acontrol{\gamma_A}
\def\dcontrol{\omega_D}

\def\ta{\tau_A(\za,\Gamma)}
\def\td{\tau_D(\zd,\Omega)}
\def\taa{\tau_A(\za,\zb)}
\def\tdd{\tau_D(\zd,\zb)}

\def\tatheta{\tau_A^\theta(\za)}
\def\tdtheta{\tau_D^\theta(\zd)}

\def\sp{\vspace{3mm}}
\def\zd{\mf z_D}
\def\za{\mf z_A}
\def\zb{\mf z_B}

\def\pp{p(\zd,\za,\zb)}
\def\p{p(\zd,\za,\Omega,\Gamma)}
\def\ptheta{p^\theta(\zd,\za,\theta)}
\def\pstar{p^*(\zd,\za)}
\def\pdstar{p(\zd,\za,\Omega^*,\Gamma)}
\def\pastar{p(\zd,\za,\Omega,\Gamma^*)}
\def\pdastar{p(\zd,\za,\Omega^*,\Gamma^*)}
\def\pthetastar{p^{\theta^*}(\zd,\za,\theta^*)}

\def\dbest{\Omega^*}
\def\abest{\Gamma^*}

\title{\LARGE \bf
Perimeter-defense Game between Aerial Defender and Ground Intruder
}

\author{Elijah S. Lee, Daigo Shishika and Vijay Kumar 
\thanks{We gratefully acknowledge the support of ARL grant DCIST CRA W911NF-17-2-0181}
\thanks{The authors are with the GRASP Lab at the University of Pennsylvania, Philadelphia, PA, 19104 USA 
        {\tt\small \{elslee, shishika, kumar\}@seas.upenn.edu}}%
}

\begin{document}

\maketitle
\thispagestyle{empty}
\pagestyle{empty}

\begin{abstract}
We study a variant of pursuit-evasion game in the context of perimeter defense. In this problem, the intruder aims to reach the base plane of a hemisphere without being captured by the defender, while the defender tries to capture the intruder. The perimeter-defense game was previously studied under the assumption that the defender moves on a circle. We extend the problem to the case where the defender moves on a hemisphere. 
To solve this problem, we analyze the strategies based on the breaching point at which the intruder tries to reach the target and predict the goal position, defined as optimal breaching point, that is achieved by the optimal strategies on both players. We provide the barrier that divides the state space into defender-winning and intruder-winning regions and prove that the optimal strategies for both players are to move towards the optimal breaching point. Simulation results are presented to demonstrate that the optimality of the game is given as a Nash equilibrium.
\end{abstract}

\section{Introduction}
The study of pursuit-evasion games (PEGs) has received interest over the past years and has played a crucial role in many different areas including missile guidance and robotics. There are many variants of PEGs under different assumptions on the players and the environments, and the surveys of the research are provided in \cite{robin2016multi, chung2011search}. 

Many researchers have focused on solving the PEGs on the planar environments where every player has motion in two dimensions \cite{liang2019differential,Zhou2016,oyler2016pursuit}. One work \cite{liang2019differential} presents a PEG with three players, target, attacker, and defender in a plane. The attacker’s goal is to capture the target without being caught by the defender, and the defender aims to defend the target while trying to capture the attacker. Zhou et al. \cite{Zhou2016} proposes cooperative pursuit of a single evader by multiple pursuers and considers the Voronoi neighbors of each player in simply connected plane. Other work \cite{oyler2016pursuit} studies PEGs in the presence of obstacles that constrain the two dimensional motions of the players.  

This work formulates a variant of pursuit-evasion game known as the target-guarding problem \cite{isaacs1999differential}. In this problem, the intruder aims to reach the target without being captured by the defender, while the defender tries to capture the intruder \cite{liang2019differential,shishika2018local,shishika2019perimeter, shishika2020cooperative}. When the defender is constrained to move along the perimeter of the target region, we call this problem as perimeter-defense game \cite{shishika2018local,shishika2019perimeter,shishika2020cooperative}. Multiplayer perimeter-defense game \cite{shishika2018local,shishika2020cooperative} and perimeter of arbitrary convex shapes \cite{shishika2019perimeter} have been studied in the past.  

All the aforementioned works deal with engagements on a planar game space. In real-world situations, target/perimeter may be close to three-dimensional shape. Accordingly, the players may be given the ability to perform three-dimensional motion, which is preferred to provide efficient and practical trajectories for such problems. As an instance, aerial vehicles \cite{lee2020experimental,nguyen2019mavnet,chen2020sloam,lee2016drone}
are viable solution to navigate through three dimension space. Many studies have focused on deploying aerial vehicles in various space such as nuclear power plant \cite{lee2020experimental}, penstock \cite{nguyen2019mavnet}, forest \cite{chen2020sloam}, or disaster sites\cite{lee2016drone}. These are good examples for perimeter-defense application and a three-dimensional target opens up the feasibility of PEGs on the real-world settings.

\begin{figure}[t]
\centering
\includegraphics[width=.49\textwidth]
{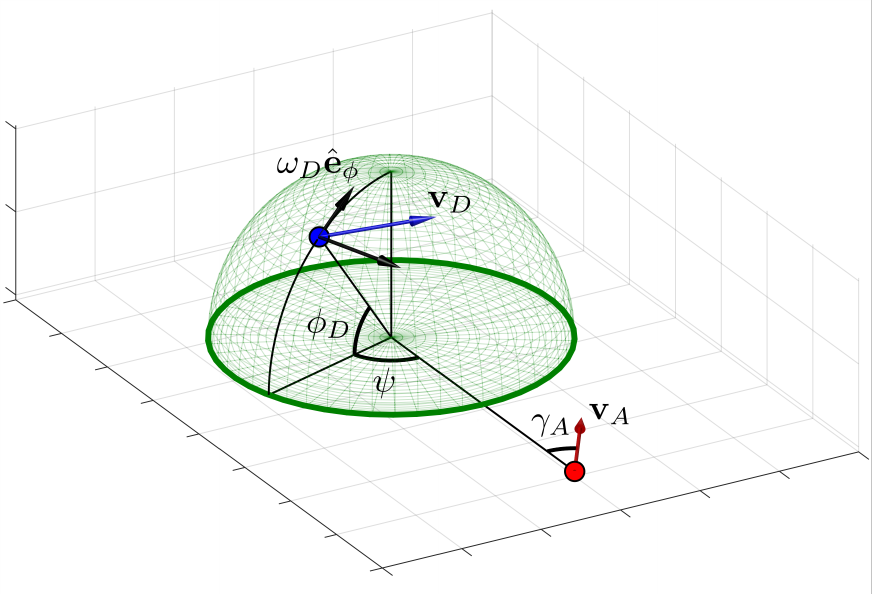}
\caption{
The coordinate system and relevant variables in the one vs one game.
}
\label{fig:hemisphere_coordinates}
\end{figure}

This paper considers three-dimensional extension of the perimeter-defense game. As an intermediate step towards a full air vs. air perimeter defense, we consider a game played between aerial defender and ground intruder. To the best of authors’ knowledge, this paper is the first to solve the pursuit-evasion game with the defender constrained to move on a hemisphere. Relevant related works by Yan et al. \cite{yan2019guarding,yan2019construction,yan2019maximum} allow the players to move in three dimensions. Yan et al. proposes  the target of a plane and solves the differential games with two defenders and one intruder \cite{yan2019guarding}, three defenders and one intruder with equal speeds \cite{yan2019construction}, and heterogeneous multiplayer \cite{yan2019maximum} in three dimensions. These solutions are reasonable for open space; however, our work aims to provide a practical solution considering real-world scenes that aerial defender has a constrained movement around the target (i.e. it cannot directly pass through the target) to capture ground intruder.

The contributions of the paper are 
(i) providing the barrier that characterizes the outcome of the game from the initial configuration; and
(ii) obtaining the players’ optimal strategies.

Section II formulates the perimeter-defense problem on a hemisphere. Section III presents candidate strategies, which is proved to be optimal in Section IV. Section V provides simulation results, and Section VI concludes the paper.
\section{Problem Formulation \label{sec:problemformulation}}
Consider two agents $A$ and $D$ denoting the intruder and the defender.
The perimeter defended by the defenders is defined as a hemisphere with unit radius.
The intruder is constrained to move on the ground plane $\mc R$, whereas the defender is constrained to move on the hemisphere.
The positions of the agents are described using spherical coordinates: $\mf z_D=[\psi_D,\phi_D,1]$ and $\mf z_A=[\psi_A,0,r]$, where $\psi$ and $\phi$ are the azimuth and elevation angles.
The relative position of the two can be described by the following states: $\mf z \triangleq [\psi,\phi,r]$, where $\psi\triangleq \psi_A-\psi_D$ and $\phi\triangleq \phi_D$ (see Fig.~\ref{fig:hemisphere_coordinates}).


We assume that all agents have first-order dynamics.
We parameterize the intruder's velocity using the heading angle $\gamma_A$ (see Fig.~\ref{fig:hemisphere_coordinates}),
where we assume $\gamma_A\in[0,\pi/2]$.
We also assume without the loss of generality that the defender's maximum speed is 1.
The intruder is assumed to have a maximum speed $\nu\leq 1$.
The defender's velocity is parameterized by the altitudinal component: $\omega_D \triangleq \dot{\phi}_D \in [-1,1]$.
Noting that the defender's speed is given by
\bql
\|\dot{\mf x}_D\|=\sqrt{\dot{\phi}_D^2+\dot{\psi}_D^2\cos^2\phi_D},
\eql
and assuming that the defender moves at its maximum (unit) speed, we have 
\bql
\dot{\psi}_D = \frac{\sqrt{1-\omega_D^2}}{\cos\phi_D}.
\eql

The state dynamics are
\begin{equation}
\dot{\mf z}=\left[ \begin{array}{c}

\dot \psi\\
\dot \phi\\
\dot r
\end{array} \right]
=
\left[ \begin{array}{c}

\ds \frac{\aspeed \sin\acontrol}{r} -\frac{\sqrt{1-\dcontrol^2}}{\cos\phi_D}\\
\dcontrol \\
-\aspeed \cos \acontrol
\end{array} \right]
=
\bs f(\mf z,\omega_D,\gamma_A).
\end{equation}
Finally, we assume complete state information, i.e., all states (positions) are known to all agents, but not the control inputs (velocities).


The game ends at time $t_f$ with intruder's win if $r(t_f)=1$ and $|\psi(t_f)| + |\phi_D(t_f)|>0$, whereas it ends with defender's win if $\phi_D(t_f)=\psi(t_f)=0$ and $r(t_f)>0$. 
We call $t_f$ as the \textit{terminal time}.
Note if the states reach the configuration $\{\mf z\,|\,\psi=\phi_D=0\}$, the defender can stabilize the states around this manifold due to its speed advantage \cite{shishika2018local}. This implies that intruder cannot reach the hemisphere without being captured by the defender.

The above defines a \textit{Game of Kind} as the question of whether intruder can reach the perimeter with non-zero terminal separation angle or the defender can drive $\psi$ and $\phi_D$ to $0$ before the intruder reaches the perimeter.
In the following sections, the surface (i.e., \emph{barrier}) separating these two cases is derived.
\section{Candidate Strategies} \label{sec:candidate}
This section discusses candidate strategies for defender and intruder. We first propose a payoff function to be used in the game of degree. Then, we derive an optimal direction of motion that maximizes the payoff function using geometric approach. 
The optimality of the candidate strategies are discussed in Sec.~\ref{sec:optimality}.

\begin{figure}[t]
\centering
\includegraphics[width=.42\textwidth]
{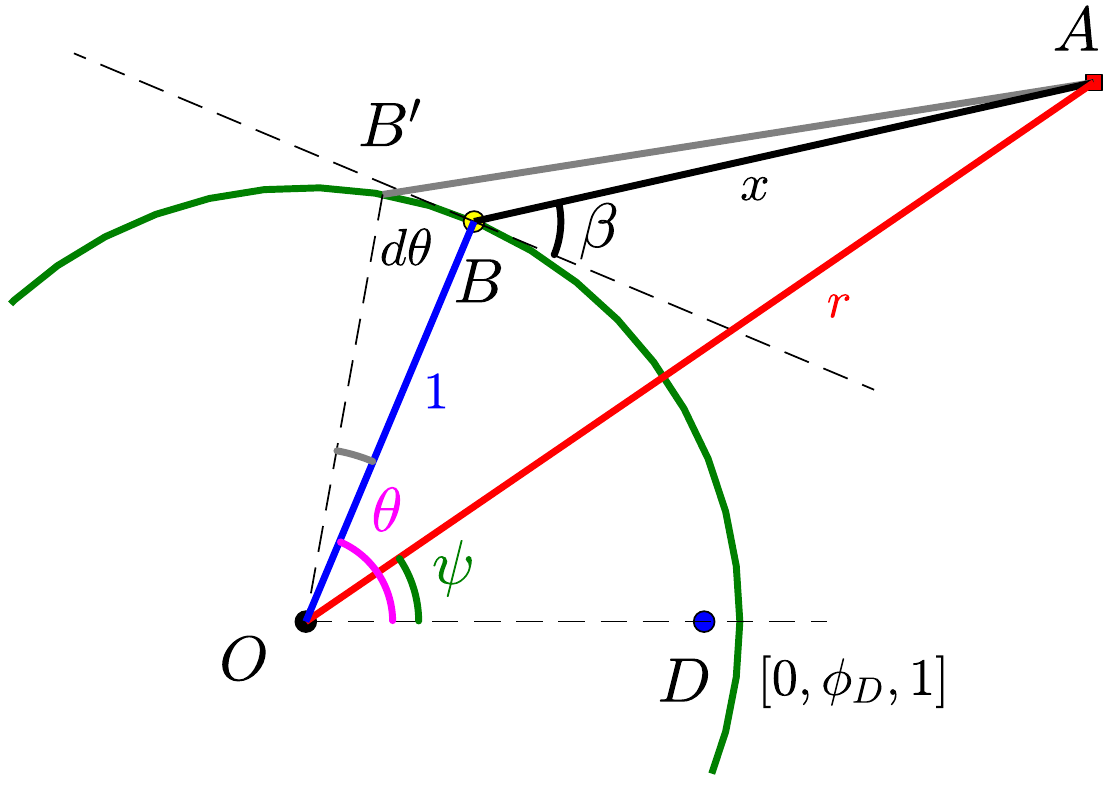}
\caption{Geometric interpretation of parameters.}
\label{fig:geometric_interpretation}
\end{figure}

\subsection{Objective function}
Given an initial configuration $\zd$ and $\za$, the goal of the intruder is to reach the perimeter. Assume the intruder reaches it at point $B$ on the ground plane of the hemisphere (See Fig. \ref{fig:geometric_interpretation}). We call $B$ as the \textit{breaching point}, $\theta\triangleq\psi_B-\psi_D$ as the \textit{breaching angle}, and $x=\lVert \mf z_A-\mf z_B \rVert$. Define the \textit{target time} as the time to go to $B$ and call $\tdd$ as the \textit{defender target time} and $\taa$ as the \textit{intruder target time}. Then, we consider the following \textit{payoff} function:
\bql
\pp = \tdd -\taa \label{eq:payoff}
\eql   

Notice the positive $p$ indicates that the intruder reaches the breaching point before the defender does and negative $p$ means vice versa. Thus, the defender intends to minimize $p$ while the intruder tries to maximize it. Now, we express the payoff $p$ in terms of control variables. Given the current states, let $\Omega$ and $\Gamma$ be the continuous control inputs of $\omega_D$ and $\gamma_A$ that lead to the breaching point. Then, \eqref{eq:payoff} becomes
\bql
\p = \td - \ta
\eql

If we call $\Omega^*$ and $\Gamma^*$ as the control inputs/strategies that minimizes $\td$ and $\ta$, respectively, the optimality in the game is given as a Nash equilibrium:
\bql
\pdstar\leq\pdastar\leq\pastar \label{eq:nash}
\eql
where the optimal payoff is given by
\beq
\pstar=\pdastar
\eeq

The defender cannot reduce $p$ by changing the strategy from $\Omega^*$, as long as the intruder sticks to its strategy $\Gamma^*$. Similarly, the intruder cannot achieve a higher $p$ by deviating from $\Gamma^*$ if the defender sticks to its strategy $\Omega^*$. 

\subsection{Candidate strategy}

Given a breaching point $B$, a candidate strategy for defender and intruder is to move towards $B$ in the shortest path. In this strategy, $\tau_D$ is the time for the defender to travel the geodesic between defender's initial position and the breaching point, which is given by
\bql
\tdd = \cos^{-1}{(\cos{\phi_D}\cos{\theta})} \label{eq:taud}
\eql
$\tau_A$ is the time for the intruder to move in a straight line towards the breaching point. Using the law of cosines from $\triangle ABO$ in Fig.~\ref{fig:geometric_interpretation}, we have
\bql
x^2 = r^2+1^2-2r\cos{(\theta-\psi)} \label{eq:x2}
\eql
which can be used with $\taa = x/\nu$ to give
\bql
\taa = \frac{1}{\nu}\sqrt{r^2+1-2r\cos{(\theta-\psi)}} \label{eq:taua}
\eql
From \eqref{eq:payoff}, \eqref{eq:taud} and \eqref{eq:taua}, we have
\bnl
\pp &=& \cos^{-1}{(\cos{\phi_D}\cos{\theta})} \\
&-&\frac{\sqrt{r^2+1-2r\cos{(\theta-\psi)}}}{\nu} \label{eq:p}
\enl

We will argue that this is the optimal payoff in Section IV.
Looking at the parameters, we observe that only $\theta$ needs to be found since all other parameters are described in the initial setup: $\mf z_D=[0,\phi_D,1]$, $\mf z_A=[\psi,0,r]$, and $\bar v_A = \nu$. 
Therefore, we define
\bql
\pp \triangleq \ptheta \triangleq \tdtheta-\tatheta \label{eq:z2theta}
\eql
such that
\bql
\tdtheta\triangleq\tdd \text{ and } \tatheta\triangleq\taa \notag
\eql
This means that a breaching angle determines a breaching point so the target time and the payoff can depend on $\theta$ instead of $\zb$. We delve into solving for $\theta$ in the next.


\subsection{Optimal breaching point}

To help solve for $\theta$, we define $\beta$ to be the angle between $(\mf z_A-\mf z_B)$ and the tangent line at the breaching point $B$ as shown in Fig.~\ref{fig:geometric_interpretation}. We call $\beta$ as the \textit{approach angle} as it determines the direction of intruder approaching to the perimeter. 

\begin{lemma} 
Suppose the positions of defender and intruder are given as $\mf z_D$ and $\mf z_A$, respectively. Then the function that maps from $\theta$ to $\beta$ is one-to-one.
\end{lemma}

\begin{proof}
By symmetry, consider breaching points $B$ at coordinate $\mf z_B=[\theta,0,1]$ so that $\theta$'s codomain is from $\psi$ to $\theta_t$, where $\theta_t$ is the angle describing a breaching point $B$ that is the point of tangency of $\overline{AB}$. It is easy to observe that as $\theta$ increases from $\psi$ to $\theta_t$, $\beta$ monotonically decreases. In Fig.~\ref{fig:geometric_interpretation}, consider that $\theta$ is increased by $d\theta$. If we call $B'$ to be the new breaching point with $d\theta$, we know $\angle AB'B<\beta$ and thus new $\beta'$ satisfies $\beta'<\beta$.  
\end{proof}

\begin{remark} 
Notice that given $\theta$ and positions of agents, $\beta$ is unique. This guarantees that for the optimal payoff, there exists a unique breaching point associated with $\theta^*$ and $\beta^*$.
\end{remark}

\begin{figure*}[!t]
\centering
\includegraphics[height=3.7cm]{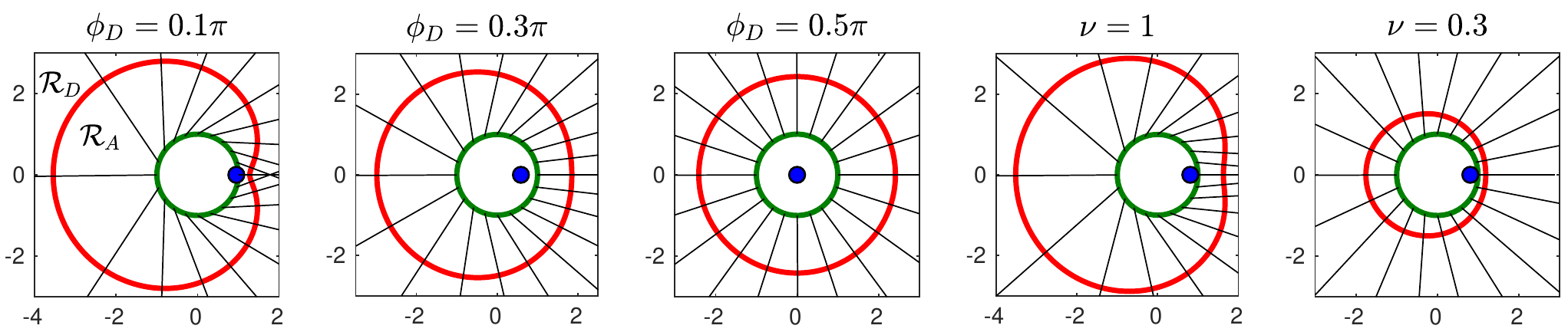}
\caption{
Instances of winning region with varying parameters. In the figure, perimeter, defender, and winning region are marked in green circle, blue dot, and red curve, respectively. Normal lines to the winning region are marked in black to help visualize possible optimal approach angles of intruder.
}
\label{fig:winning_region}
\end{figure*}

Now we continue by forming two equations relating $\theta^*$ and $\beta^*$. By Lemma 1, constructing two independent equations will solve for these parameters. The first equation comes from the optimality of the game. 

\begin{theorem}
For a given intruder and defender positions, the payoff function \eqref{eq:payoff} is maximized if the intruder selects the breaching point $\zb$ that gives the following approach angle:
\begin{equation}
\beta^* =  \cos^{-1}\left(\nu\frac{\cos{\phi_D}\sin{\theta^*}}{\sqrt{1-\cos^2{\phi_D}\cos^2{\theta^*}}}\right)
\label{eq:beta}
\end{equation}
\end{theorem}

\begin{proof}
Assuming that $p$ has an optimal value, we take a derivative of $\ptheta$ and set it to zero:
\bql
dp^\theta(\zd,\za,\theta^*) = d\tdtheta-d\tatheta = 0 \label{eq:dp}
\eql
Focusing on d$\tdtheta$ and d$\tatheta$ from \eqref{eq:dp}, we have
\bnl
d\tdtheta &=& \tau_D^{\theta+d\theta}(\zd)-\tdtheta \notag\\
&=& \cos^{-1}{(\cos{\phi_D}\cos{(\theta+d\theta)})} \notag\\
&-&\cos^{-1}{(\cos{\phi_D}\cos{\theta})}  
\label{eq:dtd}
\enl
\bql
d\tatheta = \frac{|AB'|}{\nu} - \frac{|AB|}{\nu} 
= \frac{|BB'|\cos{\beta}}{\nu} 
= \frac{d\theta\cos{\beta}}{\nu} 
\label{eq:dta}
\eql
Together with \eqref{eq:dp}, \eqref{eq:dtd} and \eqref{eq:dta}, we obtain
\bnl
\beta^* &=& \cos^{-1}\left(\nu\frac{\cos^{-1}{(\cos{\phi_D}\cos{(\theta+d\theta)})}}{d\theta}\right)\notag\\
&=& \cos^{-1}\left(\nu\frac{\cos{\phi_D}\sin{\theta^*}}{\sqrt{1-\cos^2{\phi_D}\cos^2{\theta^*}}}\right)
\enl
\end{proof}

The second equation can be obtained from the geometry. Similar to \eqref{eq:x2}, we have
\bql
r^2 = x^2+1^2-2x\cos\left(\frac{\pi}{2}+\beta\right) \label{eq:r2}
\eql
Solving for $x$ using quadratic formula, \eqref{eq:r2} gives
\bql
x = k\pm\sqrt{k^2+r^2-1}
\eql
where 
\beq
k = \cos\left(\frac{\pi}{2}+\beta\right) = -\sin\beta
\eeq
To satisfy $x>0$ given $r>1$, we take
\bnl
x &=& -\sin\beta+\sqrt{(-\sin\beta)^2+r^2-1} \notag \\
&=& -\sin\beta+\sqrt{r^2-\cos^2\beta}
\label{eq:x}
\enl
Squaring both sides of \eqref{eq:x}, we obtain
\bql
x^2 = r^2+\sin^2\beta-\cos^2\beta-2\sin\beta\sqrt{r^2-\cos^2\beta} \label{eq:x22}
\eql

\begin{remark} 
For $a\geq0$, the following is hold:
\bql
a\sin x+b\cos x = \sqrt{a^2+b^2}\sin(x+\phi)
\eql
where
\beq
\phi = \sin^{-1}\left(\frac{b}{\sqrt{a^2+b^2}}\right)
\eeq
\end{remark}

\sp
Comparing \eqref{eq:x2} and \eqref{eq:x22} with Remark 2, we get
\bnl
r\cos(\theta-\psi) &=& \sin\beta\sqrt{r^2-\cos^2\beta}+\cos^2\beta \notag\\
&=& r\sin\left(\beta+\sin^{-1}\left(\frac{\cos\beta}{r}\right)\right)
\enl
which can be simplified to give
\bql
\cos\beta = r\cos(\psi-\theta-\beta)
\eql
Solving for $\theta$, we have
\bql
\theta = \psi-\beta+\cos^{-1}\left(\frac{\cos\beta}{r}\right) \label{eq:theta}
\eql
By substituting \eqref{eq:beta} into \eqref{eq:theta}, we can solve for $\theta$.

\begin{definition} 
Suppose the position of defender $\mf z_D$ and intruder $\mf z_A$ are given. Then we define the \textit{optimal breaching angle} as $\theta^*$ that satisfies \eqref{eq:beta} and \eqref{eq:theta}. We define the corresponding \textit{optimal approach angle} as $\beta^*$ and \textit{optimal breaching point} as a breaching point that forms $\theta^*$ and $\beta^*$. 
\end{definition}


\begin{remark}
Given the position of defender $\mf z_D$ and intruder $\mf z_A$, there exists a unique pair of optimal approach angle and optimal breaching angle.
\end{remark}

The candidate optimal strategy for agents is to move towards the optimal breaching point. The payoff from this strategy is given by $\pthetastar$.

\section{Optimality Proof \label{sec:optimality}}
This section proves the optimality of the candidate strategy aforementioned in Sec. \ref{sec:candidate}. We first introduce the winning region of each agent and then prove the optimality based on the Nash equilibrium provided in \eqref{eq:nash}.

\subsection{Winning region}
We prove that the barrier (for the game of kind) is given by a simple closed curve with $\pp=0$ and characterize the winning region.

\begin{lemma} 
Given $\mf z_D$ and a breaching angle $\theta$, there exists a unique corresponding position $\mf z_A$ that makes $\ptheta = 0$.
\end{lemma}

\begin{proof}
From \eqref{eq:beta}, we first obtain a value of corresponding $\beta$. Since $\mf z_D$ is given, we can compute $\tau_D$ from \eqref{eq:taud}, which must be equal to $\tau_A$ from the condition $\ptheta = 0$ and \eqref{eq:z2theta}. Therefore, the distance between the breaching point and intruder position is $\tau_A\nu$. Knowing $\theta$, $\beta$, and distance between $\mf z_A$ and the breaching point, there exists a position of intruder $\mf z_A$. By Remark 3, there is a unique pair of $\theta$ and $\beta$ that are associated with $\mf z_A$, which guarantees the uniqueness with $\ptheta=0$ at $\mf z_A$. In this case, $\theta = \theta^*$ and $\beta = \beta^*$ corresponding to $\mf z_D$ and $\mf z_A$.
\end{proof}

\begin{lemma} 
Given $\mf z_D$, the set $\mathcal{C}(\mf z_D)=\{\mf z_A\mid p(\mf z_D, \mf z_A, \mf z_B)=0\}$ forms a simple closed curve (i.e., a connected curve that does not intersect with itself and ends at the same point where it begins).
\end{lemma} 

\begin{proof}
By Lemma 2, given $\mf z_D$ and a breaching angle $\theta$, there is a unique $\mf z_A$ with $p(\mf z_D, \mf z_A, \zb)=0$. By continuously varying $\theta$ from $0$ to $2\pi$, unique positions of $\mf z_A$ are continuously constructed. We will prove that this construction results in a simple closed curve by showing that $d\psi/d\theta>0$ so that $\mf z_A$ continuously moves to make a loop around the origin and the uniqueness of each $\mf z_A$ does not allow the curve to cross itself. 

From Fig.~\ref{fig:geometric_interpretation}, it is easy to obtain the following equation:
\bql
r\cos(\theta-\psi)=1+x\sin\beta \label{eq:rcos}
\eql
Taking the derivative of \eqref{eq:rcos} with respect to $\theta$ and rearranging the equation to solve for $d\psi/d\theta$, we have
\bql
\frac{d\psi}{d\theta} = 1+\frac{\frac{d}{d\theta}(x\sin\beta)}{r\sin(\theta-\psi)}
\eql
We know $r>0$ and $\sin(\theta-\psi)>0$. Using the facts: (i) $x=\nu\cos^{-1}(\cos\phi_D\cos\theta)$ by \eqref{eq:taud}; and (ii) $\beta$ is a function of $\theta$ by \eqref{eq:beta}, we confirm that $\frac{d}{d\theta}(x\sin\beta)$ is a function of $\theta$ and positive.
\end{proof}

\begin{lemma}[Winning region] 
For a given $\mf z_D$, define the \textit{intruder-winning region} $\mathcal{R}_A(\zd)\triangleq\{\mf z_A\mid \pp>0\}$ and the \textit{defender-winning region} $\mathcal{R}_D(\zd)\triangleq\mc R\setminus\mathcal{R}_A(\zd)$. The intruder wins against the defender if $\mf z_A(t_0)\in\mathcal{R}_A(\zd)$ and loses if $\mf z_A(t_0)\in\mathcal{R}_D(\zd)$ (see Fig.~\ref{fig:winning_region}).
\end{lemma}

\begin{proof}
Construct a simple closed curve $\mathcal{C}(\zd)$ with $\pp=0$ from Lemma 3. We argue that $\mathcal{R}_A(\zd)$ is the region inside $\mathcal{C}(\zd)$ while $\mathcal{R}_D(\zd)$ is the region outside $\mathcal{C}(\zd)$. Suppose $\mf z_A(t_0)\in\mathcal{R}_A(\zd)$. Then, the optimal breaching point $B$ can be obtained from the initial intruder location $A$. Let $C$ denote the intersection of $\overline{BA}$ and $\mathcal{C}(\zd)$. Then, $|\overline{AB}| < |\overline{BC}|$ since $\mf z_A(t_0)$ is inside $\mathcal{C}(\zd)$. This guarantees that the intruder wins against $D$ if it moves in a straight line towards $B$ because it will take less time for the intruder to reach the perimeter than for the defender. No matter what the defender does, the minimal time for it to reach $B$ is equal to the time for the intruder to travel $\overline{BC}$ since $\pp=0$ on $\mc C(\zd)$, which takes more time than to travel $\overline{AB}$.
Thus, it follows $\tau_D>\tau_A$ and $\pp>0$. 

In case $\mf z_A(t_0)\in\mathcal{R}_D(\zd)$, $\mf z_A(t_0)$ is outside $\mathcal{C}(\zd)$. In this region, the condition $\pp<0$ is hold and the intruder loses in either of the following way: (i) the intruder is outside of $\mathcal{R}_A(\zd)$ indefinitely; or (ii) it approaches to the perimeter while defender can maintain $\pp<0$ with its optimal strategy (see Sec IV.B). 
\end{proof}

The instances of winning region is shown in Fig.~\ref{fig:winning_region}. As can be seen, variation in parameters change the shape of the winning region. The first three figures in Fig.~\ref{fig:winning_region} show that the shape gets closer to a circle as $\phi_D$ increases from $0$ to $0.5\pi$, and the last two figures demonstrate that the size of winning region gets smaller with low $\nu$ because slow intruder has to be closer to the perimeter to win the game or would be caught by defender otherwise. 

\begin{figure}[t]
\centering
\includegraphics[width=8cm]{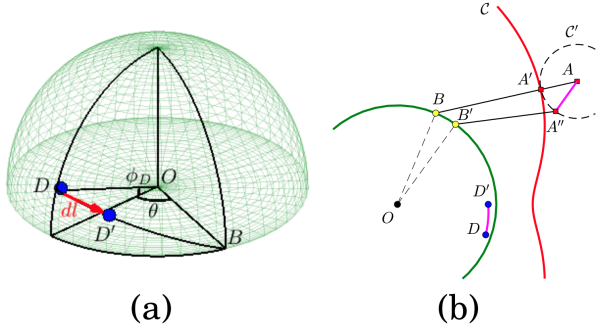}

\caption{
(a) 3D view of the game shows the defender movement along a geodesic. (b) Small movements of defender and intruder are shown in top-down view.
}
\label{fig:defender_strategy}
\end{figure}

\subsection{Proof of optimality}
To prove the optimality of the candidate strategy, we will first prove the left inequality of \eqref{eq:nash}, and then prove the right side to conclude that the candidate strategies $\Omega$ and $\Gamma$ are optimal. To prove that $\pdstar\leq\pdastar$, let the defender stick to the candidate strategy (i.e., defender always moves toward the optimal breaching point). Fig.~\ref{fig:defender_strategy} provides the geometric interpretation of such scene that the defender $D$ moves towards $B$ along a geodesic.

\begin{lemma}[Conservation of payoff] 
Given $\mf z_D$ and $\mf z_A$, if both defender and intruder move towards the optimal breaching point at their maximum speeds, $\pp$ stays the same.  
\end{lemma}


\begin{proof}
Optimal breaching point is defined as targeted position for both defender and intruder to move towards to guarantee non-changing payoff, as stated in \eqref{eq:dp}.
\end{proof}

\begin{lemma}[Conservation of optimal breaching point] 
Given $\mf z_D$ and $\mf z_A$, if both defender and intruder move towards the initial optimal breaching point at their maximum speeds, the optimal breaching point stays the same.
\end{lemma}


\begin{proof}
By Lemma 5, $\pp$ stays the same, which characterizes the optimal breaching point. By Remark 3, such optimal breaching point is unique given $\zd$ and $\za$ so the optimal breaching point stays the same. 
\end{proof}

\begin{lemma}[Degeneracy] 
Suppose $\phi_D = 0$. Without the loss of generality assume $\theta\in(0,\pi)$. Then the optimal defender strategy is to move towards the optimal breaching point at its maximum speed.
\end{lemma}

\begin{proof}
If $\phi_D = 0$, \eqref{eq:beta} becomes 
\begin{equation}
    \beta = \cos^{-1}\nu,
\end{equation}
which agrees with the results obtained in the two-dimensional version of the problem studied in \cite{shishika2018local}.
\end{proof}

Notice that if $\theta=\pi$, $\beta$ is undefined by \eqref{eq:beta}. The optimal strategy in this case is not stated in the paper but this special case will be immediately resolved by the defender's vertical motion towards the breaching point corresponding to a point at $\theta=\pi$, since $\phi_D=0$ no longer holds.


\begin{lemma} 
Given $\mf z_D$ and $\phi_D>0$, the curve $\mathcal{C}(\mf z_D)=\{\mf z_A\mid p(\mf z_D, \mf z_A, \mf z_B)=0\}$ is smooth.
\end{lemma}

\begin{proof}
To see if the curve is smooth, we find the curvature as a function of $\theta$ and observe if the value is nonzero. The curvature for a curve defined in polar coordinates is given by
\bql
\kappa(\theta)=\frac{|r^2+2r'^2-rr''|}{(r^2+r'^2)^{\frac{3}{2}}}\label{eq:k}
\eql
Since $p(\mf z_D, \mf z_A, \mf z_B)=0$, the relation \eqref{eq:r2} is valid and we can solve for $r$ to get
\bql
r=\sqrt{x^2+1+2x\sin\beta} 
\eql
Using the facts: (i) $x=\nu\cos^{-1}(\cos\phi_D\cos\theta)$ by \eqref{eq:taud}; and (ii) $\beta$ is a function of $\theta$ by \eqref{eq:beta}, we express $r$ as a function of $\theta$ and the curvature can be calculated by \eqref{eq:k}. We observe that the curvature is well defined except for when $\cos\phi_D=1$ (i.e. $\phi_D=0$).
\end{proof}

\begin{remark} 
If a curve $\mathcal{C}$ is smooth, there exists a circle $\mathcal{C'}$ with a radius $r>0$ that is tangent to $\mathcal{C}$ at any point along $\mathcal{C}$ (See Fig.~\ref{fig:levelset}(a)). Furthermore, given the circle $\mathcal{C'}$ and its tangent point $S$, there exists any other circle $\mathcal{C''}$ with a radius $r'>0$ satisfying $r>r'$ that is tangent to $\mathcal{C}$ at $S$. 
\end{remark}

\begin{lemma} 
Given $\mf z_D$, any $\mf z_A$ and corresponding $\mf z_B$ satisfy that $\mf z_A-\mf z_B$ is a normal line to the curve $\mathcal{C(\mf z_D)}=\{\mf z_A\mid p(\mf z_D, \mf z_A, \zb)=0\}$.
\end{lemma}

\begin{proof}
Notice that $\mf z_A-\mf z_B$ is the optimal direction of the defender to minimize $p(\mf z_D, \mf z_A, \zb)$ and the path satisfies \eqref{eq:dp}. Consider $\mf z_A$'s infinitesimal nearby points (i.e. one with greater $\psi$ and the other with smaller $\psi$) on $\mathcal{C}$. Then, the optimal direction $\mf z_A-\mf z_B$ must be a normal to the curve $\mathcal{C}$, otherwise one of the nearby points would provide a shorter path, which violates that it is on $\mathcal{C}(\zd)$ satisfying $p(\mf z_D, \mf z_A, \zb)=0$.
\end{proof}

\begin{figure}[t]
\centering
\includegraphics[width=8.cm]{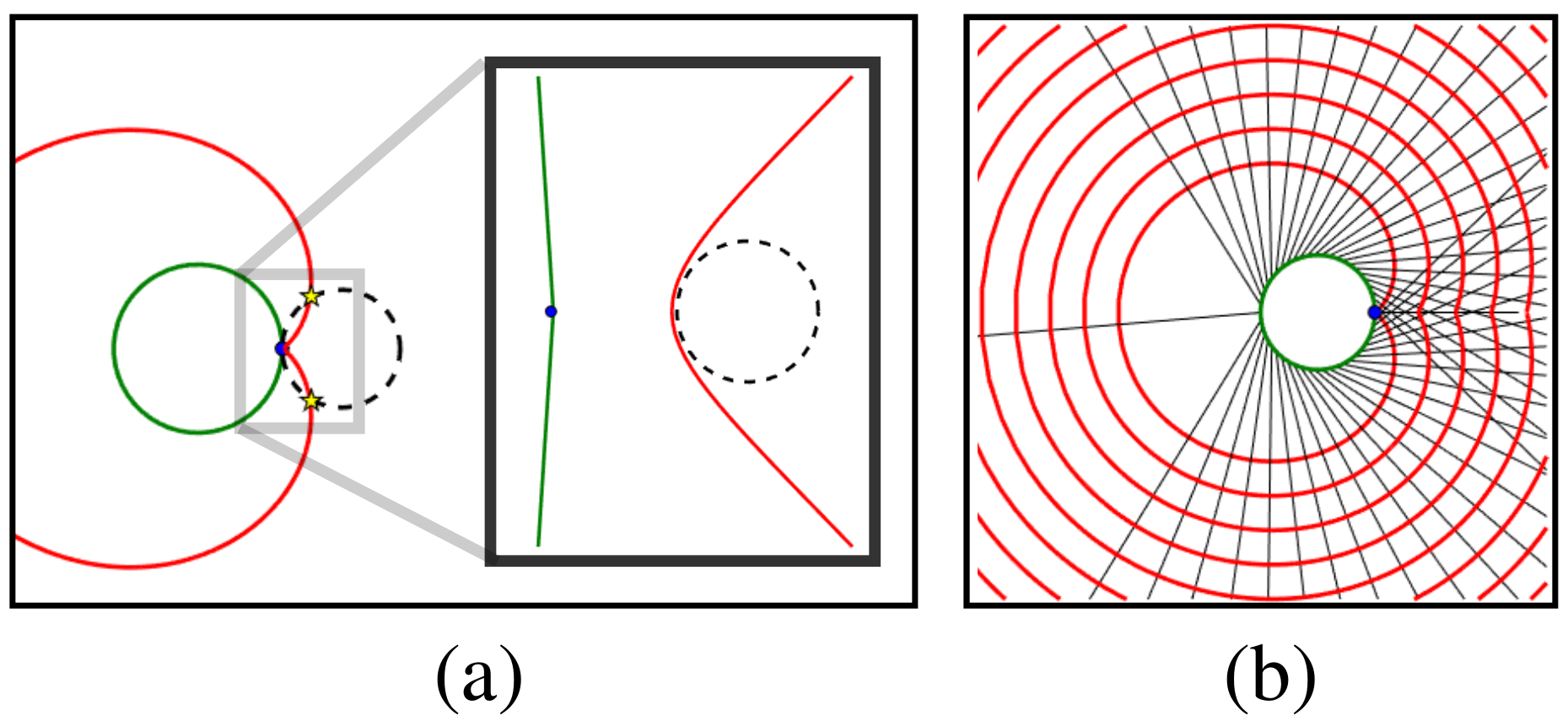}

\caption{(a) Magnified view of the barrier confirms that a circle can be tangent to a smooth curve. (b) Payoffs form a set of level sets.}
\label{fig:levelset}
\end{figure}

\begin{lemma}[Limiting case] 
Given $\mf z_D$ with $\phi_D>0$ and $\mf z_A$, such that $p(\mf z_D, \mf z_A, \zb)=0$, if defender moves towards the optimal breaching point, then $p$ is non-increasing for any intruder strategy.
\end{lemma}

\begin{proof}
Let $D,A,B$ denote the current positions of the defender, intruder, and the optimal breaching point, respectively, as shown in Fig.~\ref{fig:defender_strategy}. 
Consider an infinitesimal time $dt$ during which the defender moves towards $B$ and end up at $D'$.
Let $A'$ denote the intruder location if it moves towards the optimal breaching point at its maximum speed during this $dt$. 
By Lemma 5 and 6, the new optimal breaching point for $D'$ and $A'$ is still $B$, and $p(\mf z_{D'}, \mf z_{A'}, \zb)=0$.

Notice that if initially $\phi_D>0$, then $\phi_D=0$ will occur only when $D$ reaches $B$. Therefore, $\phi_D>0$ at $D'$.
With Lemma 8 and Remark 4, the curve $\mathcal{C}$ is smooth and there exists a circle with a radius $\bar{v}_A dt>0$ that is tangent to $\mathcal{C}$ at $A'$. 
This circle is actually centered at $A$, because by Lemma 9, $\overline{AA'}$ is a normal line to the curve $\mathcal{C}$.

Finally, by selecting sufficiently small $dt$, we can ensure that the circle does not intersect with $\mathcal{C}$.
This ensures that the circle entirely lies in the defender winning region.


Now consider any other intruder strategy that brings the intruder to point $A''$ either on or inside the circle $\mathcal{C'}$ and let $B'$ denote the optimal breaching point corresponding to $D'$ and $A''$. 
Since the circle $\mathcal{C'}$ lies in the defender winning region where $p<0$, we know that $\pp$ decreases from 0. 
Note that $p$ stays the same only if intruder continues to move towards the optimal breaching point (i.e. $A''=A'$).
\end{proof}   

\begin{lemma}[Generalization] 
Given $\mf z_D$, $\mf z_A$, and $\phi_D>0$, if defender continues to move towards the optimal breaching point, $\pp$ decreases or stays the same regardless of intruder's behavior.
\end{lemma}

\begin{proof}
By Lemma 3, the curve $\mathcal{C}=\{\mf z_A\mid p(\mf z_D, \mf z_A,\zb)=0\}$ is a simple closed curve and by Lemma 10, if intruder starts the game on $\mathcal{C}$, defender can play optimally so that $p$ does not increase. Consider level sets of $\mathcal{C}$, as shown in Fig.~\ref{fig:levelset}(b). These level sets are constructed by extending the normal lines from breaching points to $\mathcal{C}$ and connecting the points that are the same distance from $\mathcal{C}$.

Consider distinct initial intruder positions $A$ and $A'$, both on the same level set close to $\mathcal{C}$. Then $p(\mf z_D, \mf z_A,\zb)= p(\mf z_D, \mf z_{A'},\zb)$ because both intruders take the same minimal time to reach $\mathcal{C}$ where $p(\mf z_D, \mf z_A, \zb)=0$. Therefore, all the level sets can be described as a curve $\mathcal{C'}=\{\mf z_A\mid p(\mf z_D, \mf z_A,\zb)=k\}$.

Then, the same logic used to prove Lemma 10 applies. The curve $\mathcal{C'}$ is smooth since it has the same shape as $\mathcal{C}$, so Lemma 8 and Remark 4 are valid. The only difference is that $A''$ does not lie under the defender-winning region but lies under regions outside $\mathcal{C'}$, which would be on some outer level set with lower $p$. In this way, $\pp$ decreases or stays the same if intruder continues to move towards the optimal breaching point.
\end{proof}

\begin{remark}[Defender strategy] 
Optimal defender strategy is to move towards the optimal breaching point at its maximum speed at any time when $\phi_D>0$.
\end{remark}

Now that the left inequality of \eqref{eq:nash} is proven, we tackle to prove that $\pdastar\leq\pastar$.

\begin{lemma} 
Given $\mf z_D$ and $\mf z_A$, if intruder continues to move towards the optimal breaching point, $\pp$ increases or stays the same regardless of defender's behavior.
\end{lemma}

\begin{proof}
We know that $p$ will not decrease if intruder continues to move towards the optimal breaching point at its maximum speed because from Lemma 11, optimal defender strategy is to move towards the optimal breaching point as well and $p$ stays the same if both defender and intruder move towards the optimal breaching point at their maximum speeds by Lemma 5. If we assume that an arbitrary defender movement decreases $p$ when the intruder continue to move towards the optimal breaching point at its maximum speed, then it will violate the optimal behavior of the defender because the arbitrary move results in lower $p$. Furthermore, $p$ can increase if the defender randomly moves and delays capturing the intruder (i.e. $\tau_D$ indefinitely increases). 
\end{proof}

\begin{remark}[Intruder strategy] 
Optimal intruder strategy is to move towards the optimal breaching point at its maximum speed at any time.
\end{remark}

\section{Simulation}
We run simulations to demonstrate the optimality of the game. We prepare three different setups: (i) both defender and intruder follow their optimal strategies $\dbest$ and $\abest$; (ii) only defender follows $\dbest$; and (iii) only intruder follows $\abest$. The experiments are run with an initial configuration $\mf z=[\psi,\phi,r]=[0.9,0.3\pi,2]$, and intruder maximum speed $\nu=0.8$. Fig.~\ref{fig:traj} shows the simulation results for the three setups. The top plots show $\td$, $\ta$, and $\p$ for each condition, and the bottom figures show corresponding defender and intruder trajectories.

In Fig.~\ref{fig:traj}(a), we observe that the payoff $\pdastar$ remains the same, and this is expected since all the players are following their optimal strategies by moving towards the optimal breaching point in the shortest distance. Fig.~\ref{fig:traj}(b) shows that payoff $\pdstar$ is non-increasing since the defender always moves towards the optimal breaching point while the intruder moves in an arbitrary direction. In this case, the game ends with defender's win. Fig.~\ref{fig:traj}(c) displays the payoff $\pastar$ is non-decreasing. The intruder successfully enters the perimeter regardless of defender's behavior. 

The simulation results demonstrate that the optimality of the game is given as a Nash equilibrium: $\pdstar\leq\pdastar\leq\pastar$.

\begin{figure*}[!t]
\centering
\includegraphics[height=4.5cm]{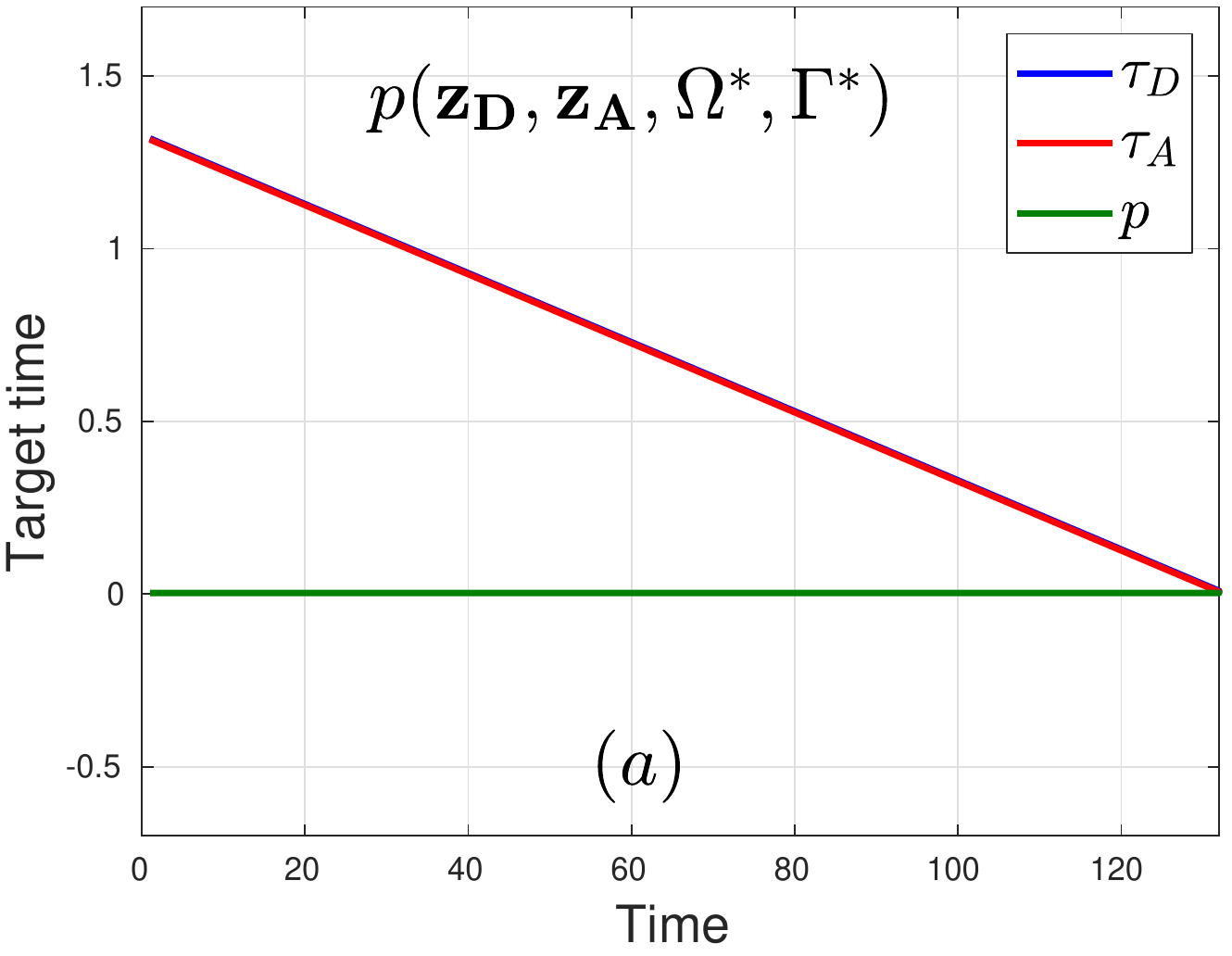}
\includegraphics[height=4.5cm]{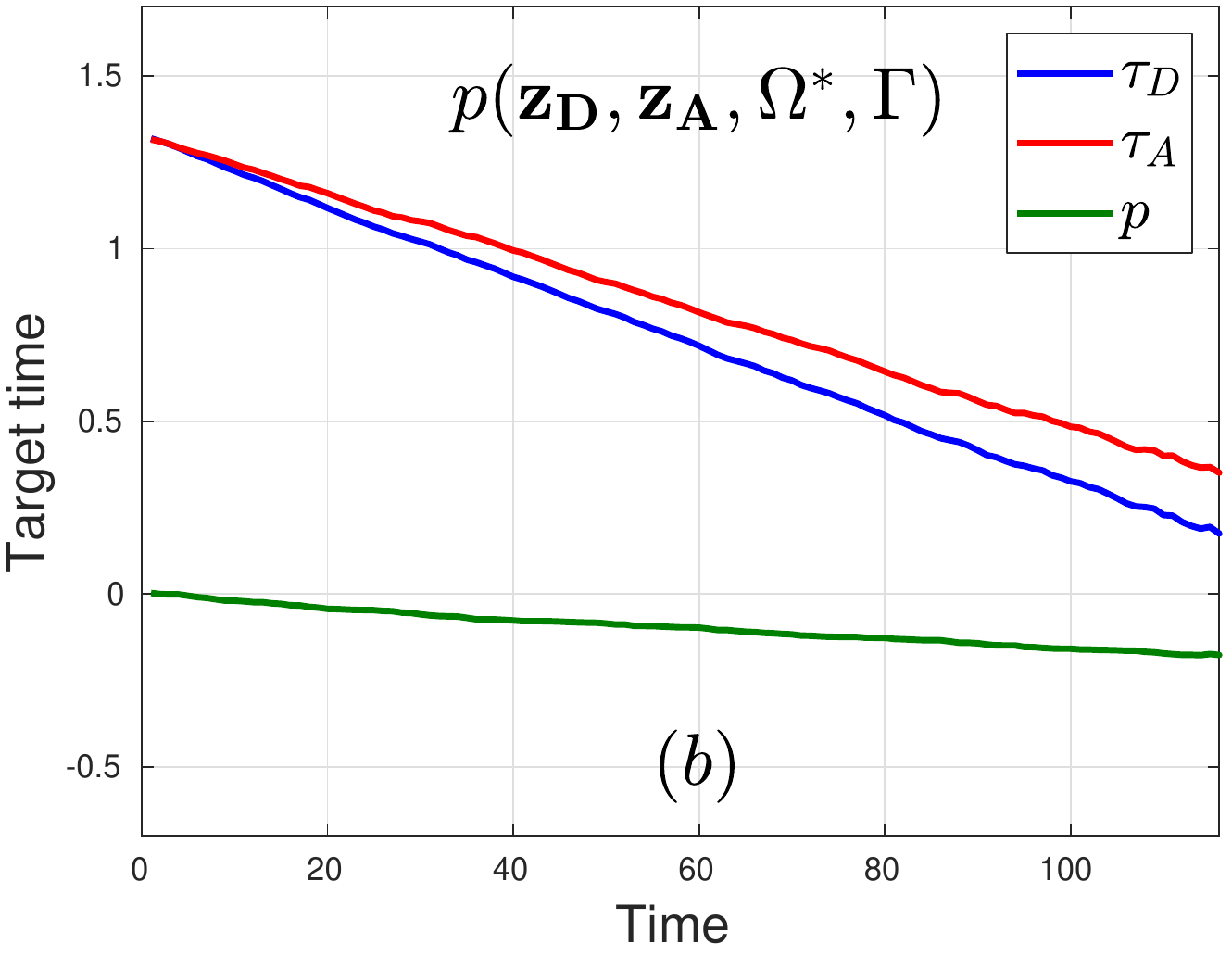}
\includegraphics[height=4.5cm]{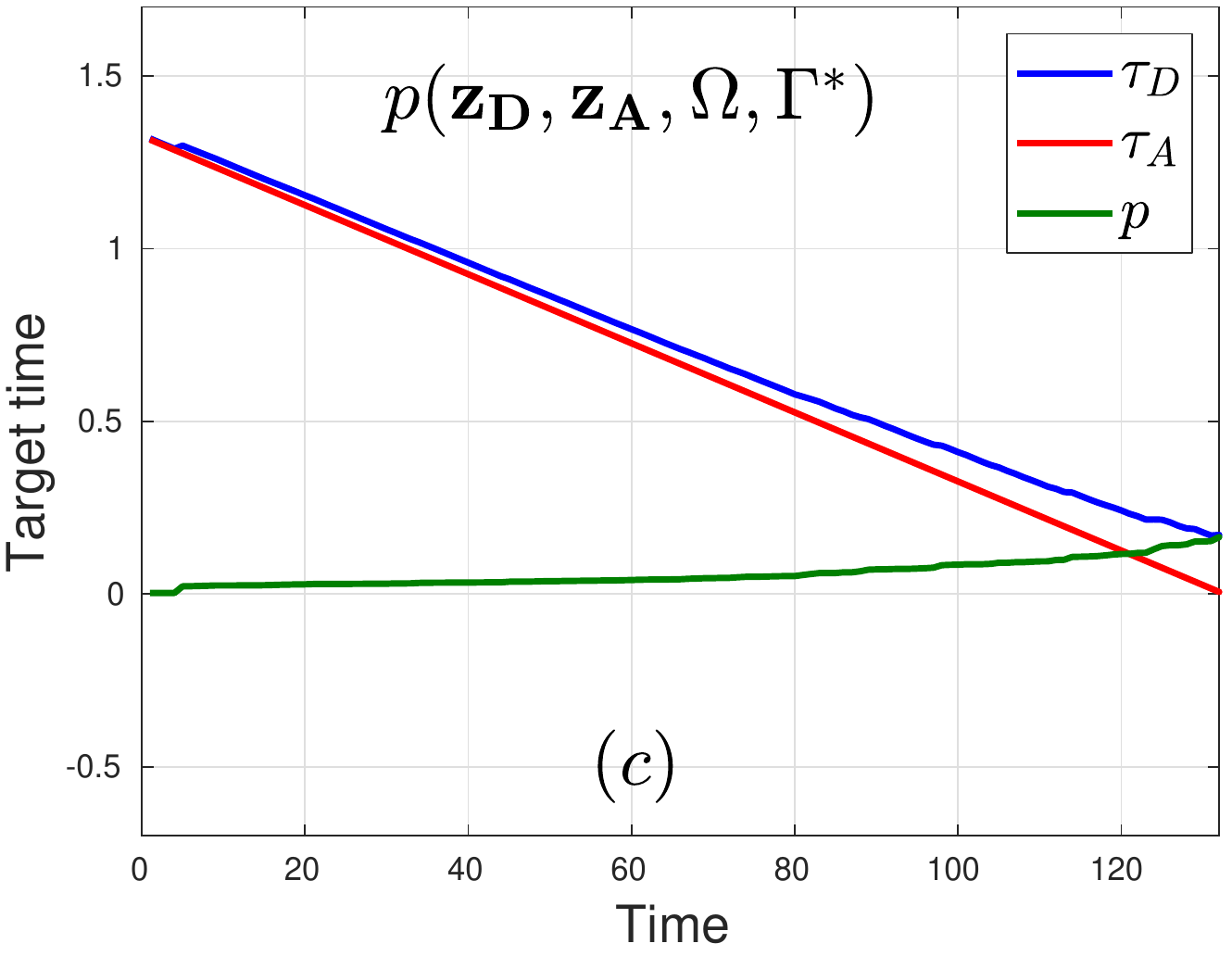}

\includegraphics[height=5.6cm]{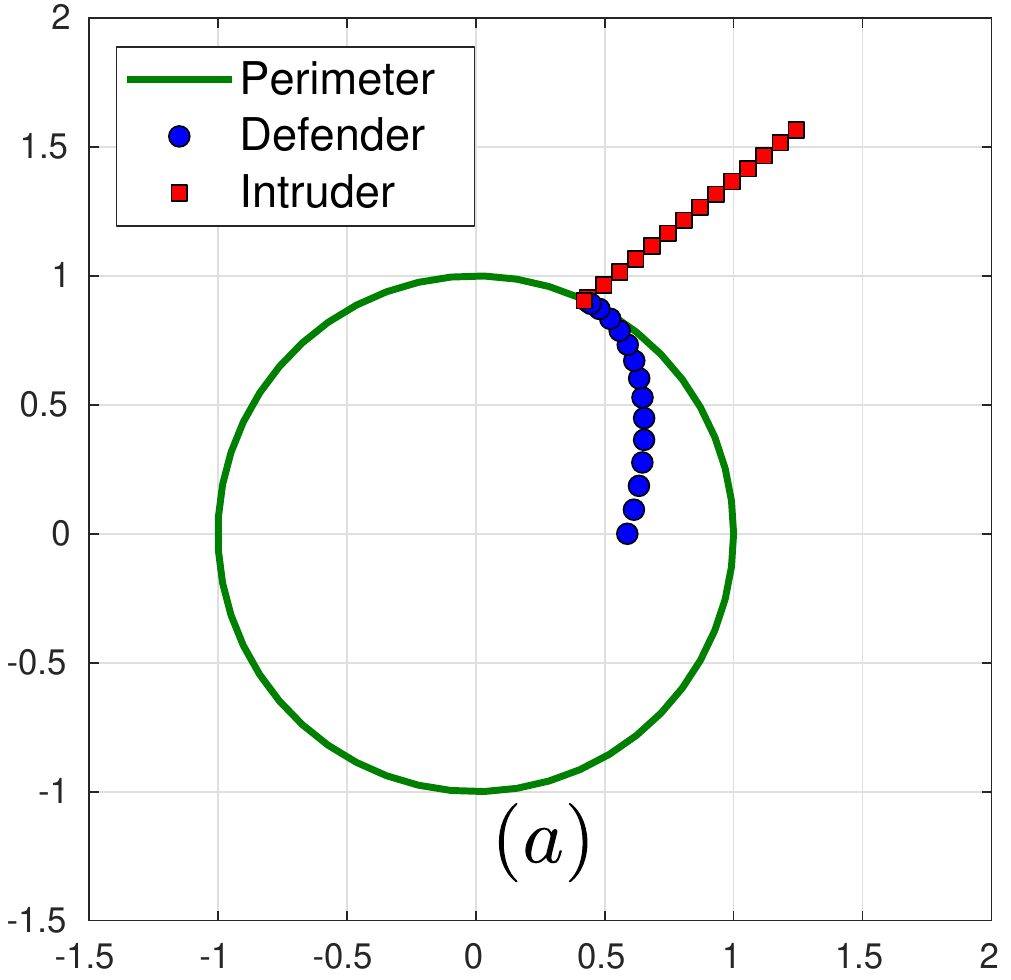}
\includegraphics[height=5.6cm]{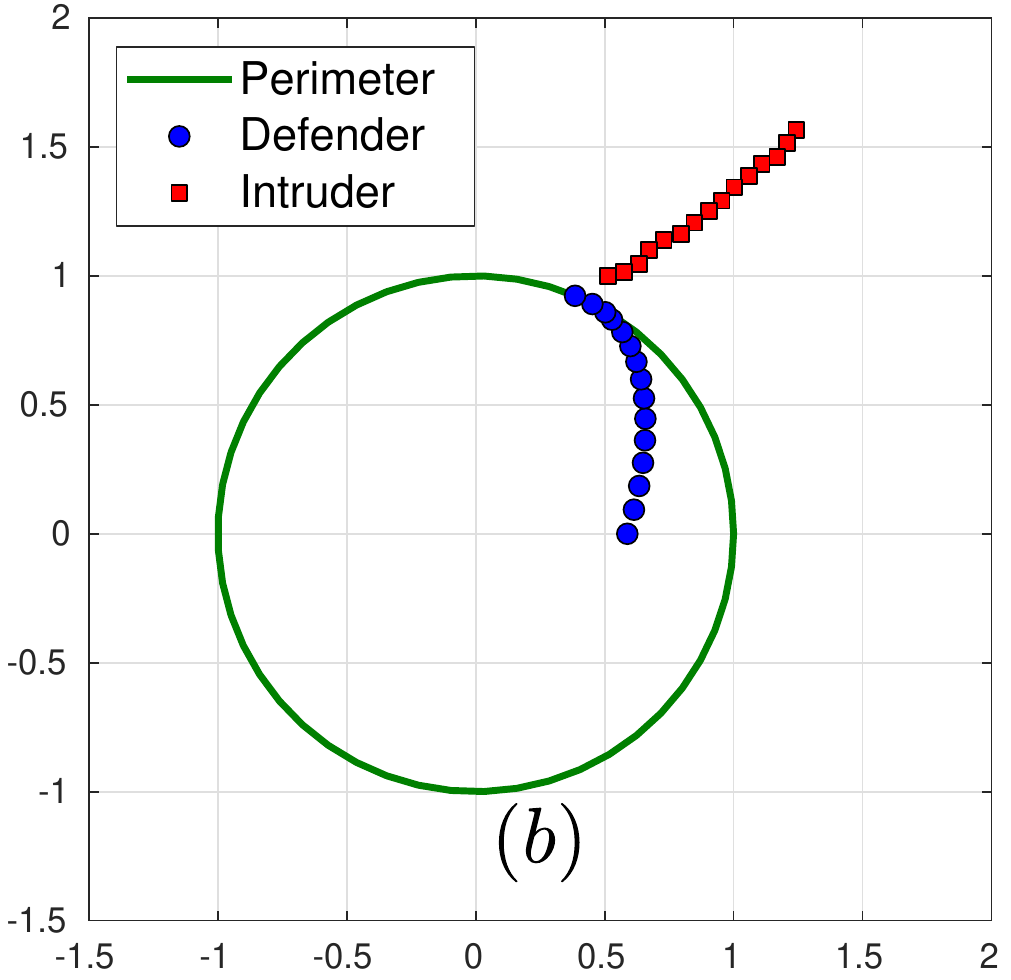}
\includegraphics[height=5.6cm]{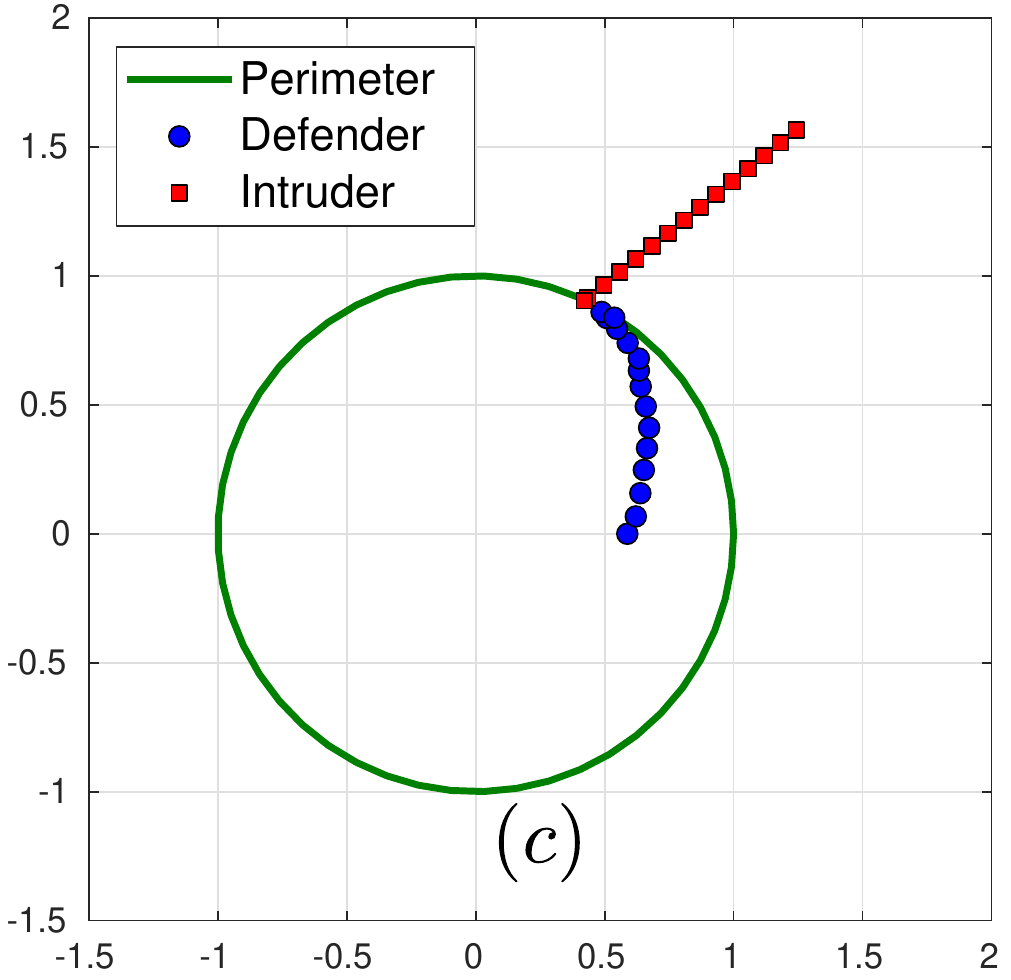}

\caption{
Simulation results for hemisphere defense game. The top figures display payoff and target time, and the bottom shows the corresponding trajectories of defender and intruder when (a) both both defender and intruder execute their optimal strategies, (b) only defender behaves optimally, and (c) only intruder follows its optimal strategy.}
\label{fig:traj}
\end{figure*}

\section{Conclusion \label{sec:conclusion}}
This paper discusses an approach to solve the perimeter-defense game on a hemisphere. To solve for the optimal strategies, we first propose an objective function called payoff and introduce candidate optimal strategies. We take a geometric approach to characterize the barrier that divides the defender and intruder winning regions. For the optimality proof, we aim to prove both sides of inequalities from \eqref{eq:nash} and we confirm that the strategies $\Omega^*$ and $\Gamma^*$ for both defender and intruder are to move towards the optimal breaching point at their maximum speeds at any time. The simulation verifies the optimality of the game as a Nash equilibrium. The future work will aim to solve the perimeter-defense problem between aerial defender and aerial intruder.






%
%
%


\bibliography{main}

\begin{thebibliography}{10}

\bibitem{robin2016multi}
Cyril Robin and Simon Lacroix.
\newblock Multi-robot target detection and tracking: taxonomy and survey.
\newblock {\em Autonomous Robots}, 40(4):729--760, 2016.

\bibitem{chung2011search}
Timothy~H Chung, Geoffrey~A Hollinger, and Volkan Isler.
\newblock Search and pursuit-evasion in mobile robotics.
\newblock {\em Autonomous robots}, 31(4):299, 2011.

\bibitem{liang2019differential}
Li~Liang, Fang Deng, Zhihong Peng, Xinxing Li, and Wenzhong Zha.
\newblock A differential game for cooperative target defense.
\newblock {\em Automatica}, 102:58--71, 2019.

\bibitem{Zhou2016}
Zhengyuan Zhou, Wei Zhang, Jerry Ding, Haomiao Huang, Du{\v{s}}an~M.
  Stipanovi{\'{c}}, and Claire~J. Tomlin.
\newblock {Cooperative pursuit with Voronoi partitions}.
\newblock {\em Automatica}, 72:64--72, 2016.

\bibitem{oyler2016pursuit}
Dave~W Oyler, Pierre~T Kabamba, and Anouck~R Girard.
\newblock Pursuit--evasion games in the presence of obstacles.
\newblock {\em Automatica}, 65:1--11, 2016.

\bibitem{isaacs1999differential}
Rufus Isaacs.
\newblock {\em Differential games: a mathematical theory with applications to
  warfare and pursuit, control and optimization}.
\newblock Courier Corporation, 1999.

\bibitem{shishika2018local}
Daigo Shishika and Vijay Kumar.
\newblock Local-game decomposition for multiplayer perimeter-defense problem.
\newblock In {\em 2018 IEEE Conference on Decision and Control (CDC)}, pages
  2093--2100. IEEE, 2018.

\bibitem{shishika2019perimeter}
Daigo Shishika and Vijay Kumar.
\newblock Perimeter-defense game on arbitrary convex shapes.
\newblock {\em arXiv preprint arXiv:1909.03989}, 2019.

\bibitem{shishika2020cooperative}
Daigo Shishika, James Paulos, and Vijay Kumar.
\newblock Cooperative team strategies for multi-player perimeter-defense games.
\newblock {\em IEEE Robotics and Automation Letters}, 5(2):2738--2745, 2020.

\bibitem{lee2020experimental}
Elijah~S Lee, Giuseppe Loianno, Dinesh Thakur, and Vijay Kumar.
\newblock Experimental evaluation and characterization of radioactive source
  effects on robot visual localization and mapping.
\newblock {\em IEEE Robotics and Automation Letters}, 5(2):3259--3266, 2020.

\bibitem{nguyen2019mavnet}
Ty~Nguyen, Shreyas~S Shivakumar, Ian~D Miller, James Keller, Elijah~S Lee, Alex
  Zhou, Tolga {\"O}zaslan, Giuseppe Loianno, Joseph~H Harwood, Jennifer
  Wozencraft, et~al.
\newblock Mavnet: An effective semantic segmentation micro-network for
  mav-based tasks.
\newblock {\em IEEE Robotics and Automation Letters}, 4(4):3908--3915, 2019.

\bibitem{chen2020sloam}
Steven~W Chen, Guilherme~V Nardari, Elijah~S Lee, Chao Qu, Xu~Liu, Roseli
  Ap~Francelin Romero, and Vijay Kumar.
\newblock Sloam: Semantic lidar odometry and mapping for forest inventory.
\newblock {\em IEEE Robotics and Automation Letters}, 5(2):612--619, 2020.

\bibitem{lee2016drone}
Seoungjun Lee, Dongsoo Har, and Dongsuk Kum.
\newblock Drone-assisted disaster management: Finding victims via infrared
  camera and lidar sensor fusion.
\newblock In {\em 2016 3rd Asia-Pacific World Congress on Computer Science and
  Engineering (APWC on CSE)}, pages 84--89. IEEE, 2016.

\bibitem{yan2019guarding}
Rui Yan, Zongying Shi, and Yisheng Zhong.
\newblock Guarding a subspace in high-dimensional space with two defenders and
  one attacker.
\newblock {\em arXiv preprint arXiv:1904.01113}, 2019.

\bibitem{yan2019construction}
Rui Yan, Zongying Shi, and Yisheng Zhong.
\newblock Construction of the barrier for reach-avoid differential games in
  three-dimensional space with four equal-speed players.
\newblock In {\em 2019 IEEE 58th Conference on Decision and Control (CDC)},
  pages 4067--4072. IEEE, 2019.

\bibitem{yan2019maximum}
Rui Yan, Xiaoming Duan, Zongying Shi, Yisheng Zhong, and Francesco Bullo.
\newblock Maximum-matching capture strategies for 3d heterogeneous multiplayer
  reach-avoid games.
\newblock {\em arXiv preprint arXiv:1909.11881}, 2019.

\end{thebibliography}
\bibliographystyle{unsrt}

\end{document}